\documentclass{article} 
\usepackage{arxiv,times}

\usepackage{amsmath,amsfonts,bm}

\def\eqref#1{equation~\ref{#1}}

\def\1{\bm{1}}

\DeclareMathAlphabet{\mathsfit}{\encodingdefault}{\sfdefault}{m}{sl}
\SetMathAlphabet{\mathsfit}{bold}{\encodingdefault}{\sfdefault}{bx}{n}

\usepackage{wrapfig}
\usepackage{hyperref}
\usepackage{url}
\usepackage{enumitem}
\usepackage{graphicx}
\usepackage{booktabs}
\usepackage{makecell}
\usepackage{multirow}
\usepackage{capt-of}
\usepackage{placeins}
\usepackage{verbatim}
\usepackage{amsthm}
\usepackage{amssymb}
\newtheorem{theorem}{Theorem}
\usepackage{amsmath}
\usepackage{mathtools}
\usepackage{needspace}

\title{Many Ways to Succeed: Diversity-Driven RL Fine-Tuning for VLA Generalization}

\author{
\begin{tabular}{@{}llll@{}}
\textbf{Haoru Li}\textsuperscript{1$\ast$} &
\textbf{Jinmei Liu}\textsuperscript{1$\ast$} &
\textbf{Zhiyong Wang}\textsuperscript{2} &
\textbf{Xiaoming Li}\textsuperscript{1}
\\
\textbf{Zhenhong Sun}\textsuperscript{3} &
\textbf{Daoyi Dong}\textsuperscript{4} &
\textbf{Chunlin Chen}\textsuperscript{1} &
\textbf{Zhi Wang}\textsuperscript{1}
\end{tabular}
\\
\textsuperscript{1} Nanjing University
\hspace{6.3em}
\textsuperscript{2} Harbin Institute of Technology (Shenzhen)
\\
\textsuperscript{3} Australian National University
\qquad
\textsuperscript{4} University of Technology Sydney
\\
\texttt{\{haoruli,jmliu\}@smail.nju.edu.cn}
\qquad
\texttt{zhiwang@nju.edu.cn}
}
\iclrfinalcopy 
\begin{document}

\maketitle

\renewcommand{\thefootnote}{}
\footnotetext{
$^\ast$ Equal contribution.
}
\renewcommand{\thefootnote}{\arabic{footnote}}

\addtocontents{toc}{\protect\setcounter{tocdepth}{-1}}

\begin{abstract}
Reinforcement learning (RL) fine-tuning improves vision-language-action (VLA) policies through closed-loop experience, yet generalization beyond the fine-tuning distribution remains limited. Our analysis reveals a selective reshaping of exploration: RL contracts behavior globally, yet diversifies successful trajectories, elicits success with fewer rollouts, and covers more of the latent task-valid solution space than supervised fine-tuning. Broader successful-mode coverage may provide alternative strategies under distribution shifts. Inspired by this, we introduce \textbf{DRIVE} (\textbf{D}iversity-driven \textbf{R}L f\textbf{I}ne-tuning for \textbf{V}LA g\textbf{E}neralization), which turns successful-behavior diversity into an explicit RL objective. DRIVE groups rollouts under matched task conditions, compares their trajectories with temporal alignment, and derives a success-conditioned intrinsic reward from relative behavioral diversity. This design encourages broader coverage of feasible solutions without rewarding diverse failures or superficial timing differences. Across LIBERO-Plus, ManiSkill3, and RoboTwin~2.0, DRIVE improves the average out-of-domain (OOD) performance over vanilla RL fine-tuning by 5.3 points on \(\pi_0\) and 2.0 points on \(\pi_{0.5}\). On a dual-arm AgileX PiPER-X platform, DRIVE further increases average OOD success from 64.1\% to 73.3\% (+9.2 points), demonstrating gains that persist under physical deployment.
\end{abstract}


\section{Introduction}
\label{sec:introduction}

\begin{wrapfigure}{r}{0.45\columnwidth}
    \centering
    \vspace{-8pt}
    \includegraphics[
        width=0.45\columnwidth
    ]{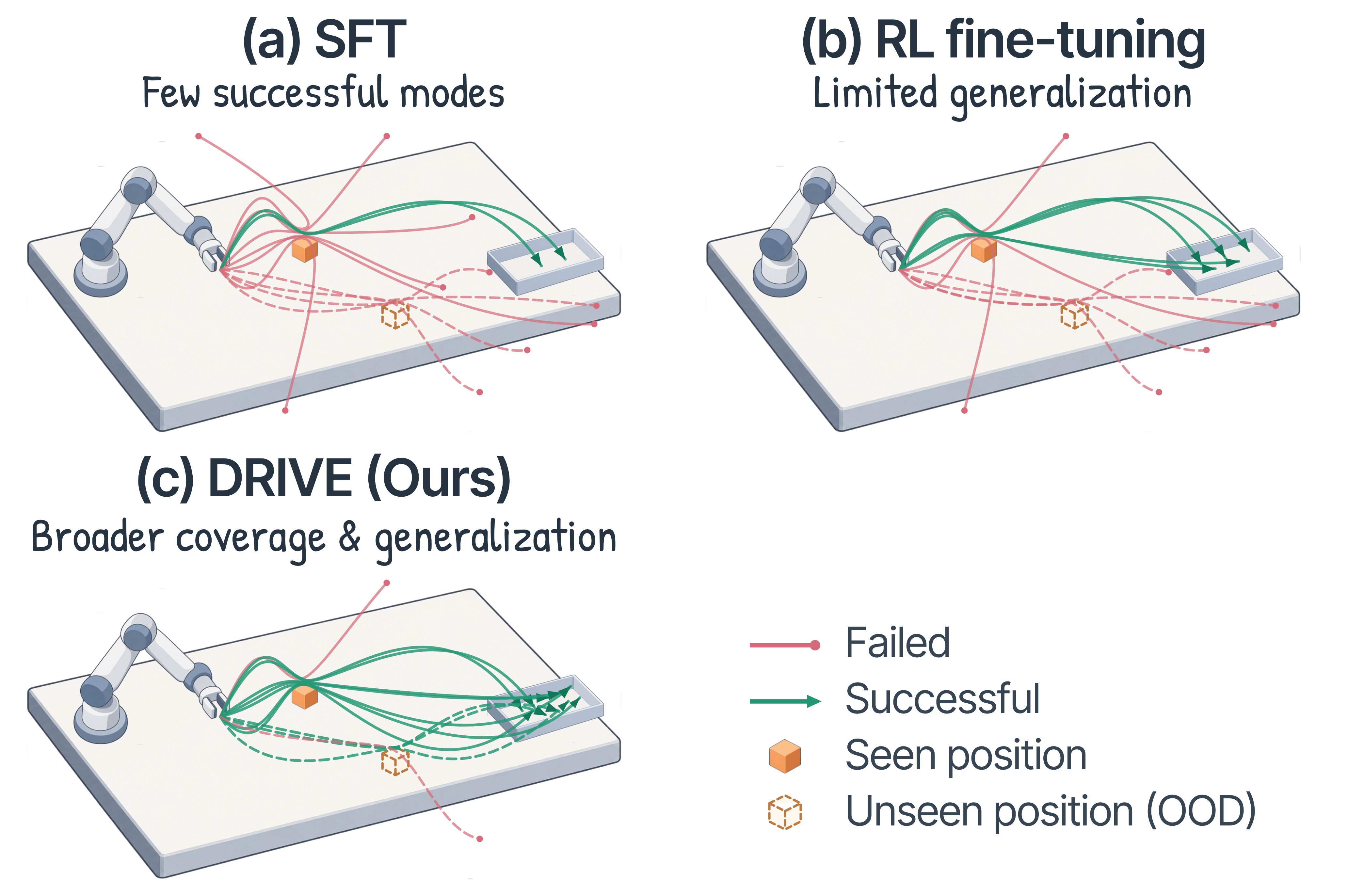}
    \caption{
    Beyond task success, diverse successful trajectories reveal broader solution modes.
    This diversity provides a new opportunity for improving VLA generalization.
    }
    \label{fig:teaser}
    \vspace{-8pt}
\end{wrapfigure}

Recent advances in vision-language-action (VLA) models have substantially expanded the range of manipulation tasks that robots can perform from visual observations and language instructions~\citep{kim2024openvla}. 
VLA models are trained primarily through supervised fine-tuning (SFT) on robot demonstrations, acquiring broad manipulation skills by imitating a fixed behavior distribution~\citep{intelligence2025pi05}. 
Reinforcement fine-tuning (RFT) complements this paradigm with self-generated experience, using outcome feedback to explore beyond demonstrations and optimize closed-loop task performance~\citep{intelligence2025pi}.
Recent studies have shown that RL fine-tuning, particularly with PPO, can improve VLA generalization in semantic understanding and execution robustness~\citep{liu2025rlvla}. 
However, such gains do not transfer uniformly across distribution shifts, and VLA policies remain sensitive to changes in visual appearance, scene configuration, and execution conditions~\citep{fei2025liberoplus}. 
Robust VLA generalization is essential for real-world deployment, yet remains challenging beyond the fine-tuning distribution. 
To understand how RFT reshapes the range of behaviors available to VLA policies, we analyze both how rollout distributions evolve and how readily their successful modes can be elicited (see Sec.~\ref{sec:exploration_dynamics}).
As training progresses, similarity increases across \textit{all} rollouts but decreases among \textit{successful} ones, revealing global behavioral contraction alongside successful-mode diversification.  
Pass@k measures how readily any successful behavior can be elicited, showing that RFT requires far fewer rollouts than SFT to obtain a success. 
Coverage@k instead measures the breadth of these successes, showing that, under the same rollout budget, RFT covers a substantially broader portion of the latent task-valid solution space than SFT.
Together, these results suggest that RFT makes task-valid solutions more accessible even as the overall behavior distribution contracts. Broader coverage of this space may provide alternative strategies under distribution shifts, revealing \textit{an underexplored opportunity to improve generalization through successful-behavior diversity}.
 
Turning this opportunity into a useful training objective is nontrivial. Standard binary task rewards distinguish success from failure but not among different successful solutions, providing no explicit incentive to maintain broad successful-mode coverage. 

Although diversity has been studied in unsupervised skill discovery~\citep{eysenbach2019diayn}, policy diversification~\citep{masood2019dipg}, and embodied exploration~\citep{chen2026curiousvla}, directly applying these approaches to VLA fine-tuning can reward superficial differences caused by execution speed, temporal misalignment, or irrelevant motion, as well as noisy and unsuccessful behaviors. 
A suitable diversity objective must therefore compare trajectories under matched task conditions, account for temporal variation, and reward behavioral differences only when the trajectories remain task-valid.

Motivated by these observations, we propose \textbf{DRIVE} (\textbf{D}iversity-driven \textbf{R}L f\textbf{I}ne-tuning for \textbf{V}LA g\textbf{E}neralization), a framework that turns successful-mode diversity into an explicit optimization target for promoting generalization.
DRIVE groups rollouts collected under matched task conditions, quantifies their behavioral novelty using temporally aligned GAK (global alignment kernel) similarity, and assigns a relative novelty bonus only to successful trajectories.
This design focuses the diversity incentive on distinct task-valid solutions rather than arbitrary variation across tasks, timing, or failures. 
Then, we formulate the intra-group diversity as an intrinsic reward and design a potential-based reward shaping scheme to preserve optimal policy invariance. 
By encouraging broader coverage of successful behaviors during fine-tuning, DRIVE provides alternative task-solving strategies that may transfer under distribution shifts. 
Across LIBERO-Plus, ManiSkill3, and RoboTwin~2.0, DRIVE improves 11 of 12 benchmark-split results on each backbone, raising the OOD macro-average from 48.1\% to 53.4\% for \(\pi_0\) and from 71.1\% to 73.1\% for \(\pi_{0.5}\). On a dual-arm AgileX PiPER-X platform, DRIVE further increases average OOD success across two tasks and three distribution shifts from 64.1\% to 73.3\% (+9.2 points).
Together, these results validate successful-mode diversity as a practical objective across simulated distribution shifts and physical deployment.

In summary, our main contributions are threefold:
\begin{itemize}[itemsep=0.25em, leftmargin=1.25em]\vspace{-0.5em}

    \item We show that global behavioral contraction coexists with successful-mode diversification, revealing successful-behavior diversity as an underexplored route to VLA generalization.

    \item We introduce DRIVE, which converts temporally aligned comparisons of grouped trajectories into a success-conditioned intrinsic reward for diverse task-valid exploration.

    \item We demonstrate consistent generalization gains across LIBERO-Plus, ManiSkill, and RoboTwin, and confirm that DRIVE’s improvements persist under physical deployment.
\end{itemize}

\section{ Related Work}
\label{sec:related_work}

\textbf{VLA Models and Generalization.}
VLA models map visual observations and language instructions to robot actions, typically combining pretrained vision-language representations with supervised learning from robot demonstrations~\citep{kim2024openvla,black2025pi0,xie2026look}. However, systematic evaluations reveal sensitivity to changes in visual appearance, scene configurations, robot initial states, and execution conditions~\citep{fei2025liberoplus,liu2025rlvla}. To improve generalization, prior work has explored data augmentation and domain randomization~\citep{chen2025robotwin,xue2025resample}, large-scale pretraining and heterogeneous co-training~\citep{zitkovich2023rt2,intelligence2025pi05}, and model designs that strengthen spatial representations and embodied reasoning~\citep{qu2025spatialvla,zawalski2024ecot}. These efforts primarily focus on training data and model capabilities, with less attention to how behavior distributions shaped during RFT affect generalization.

\textbf{Reinforcement Fine-Tuning for VLA.}
RFT complements demonstration-based SFT with self-generated experience and task-level feedback from closed-loop interaction, with successful deployments in LLMs~\citep{guo2025deepseek,yan2025learning,zhan2026exgrpo}, VLMs~\citep{liu2025visual,hu2026bridging}, and VLAs~\citep{intelligence2025pi,liu2025rlvla}. 
Recent PPO-based~\citep{lu2025vla,tan2025interactive} and GRPO-based methods~\citep{li2026simplevla} use this feedback to optimize task-completion or progress rewards. Empirical studies report that such fine-tuning improves task success over SFT and can enhance OOD generalization, particularly in semantic understanding and execution robustness~\citep{liu2025rlvla,chen2025pirl}. 
However, this line of work primarily examines the performance benefits of task-reward optimization, with less attention to how successful-behavior distributions evolve during RFT or how broader coverage of the successful behavior space can be explicitly promoted.
\textbf{Behavioral Diversity and Exploration.}
Unsupervised skill discovery promotes diverse behaviors without task rewards~\citep{eysenbach2019diayn,sharma2020dads}, while intrinsic-motivation methods encourage exploration through prediction-based novelty rewards~\citep{pathak2017curiosity,burda2019rnd}. 

Task-aware methods instead combine return maximization with differences between policy-induced trajectory distributions~\citep{masood2019dipg} or encourage diversity among near-optimal behaviors~\citep{kumar2020smerl,hu2026diversity,liu2026beyond}.
In VLA learning, diversity has been used to guide exploration through trajectory expansion and behavior conditioning~\citep{chen2026curiousvla,kong2026extoken}, and to enrich pretraining data with diverse successful trajectories~\citep{yang2025dlr}. 
In contrast, DRIVE applies a success-conditioned diversity reward across rollouts under matched task conditions, targeting broader task-valid solution coverage rather than global behavioral variation.

\section{Exploration Dynamics in VLA Reinforcement Fine-Tuning}
\label{sec:exploration_dynamics}

We study how RFT reshapes the diversity and accessibility of successful behaviors in \(\pi_0\) and \(\pi_{0.5}\) on ManiSkill3~\citep{taomaniskill3}, evaluating the SFT and intermediate/final RFT checkpoints. Using SDE sampling under matched task conditions, we measure trajectory similarity separately across all and successful rollouts. For the sampling-budget analysis, we evaluate SFT and final-RFT \(\pi_0\) on 32 task-initial-state configurations, sampling up to 256 rollouts per policy and configuration.

\begin{figure*}[t]
    \centering
    \includegraphics[
    width=\textwidth,
    trim=0 5pt 0 0pt,
    clip
    ]{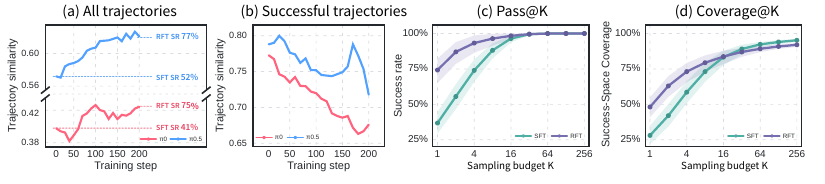}
    \caption{
    Exploration dynamics and successful-behavior accessibility in VLA reinforcement fine-tuning.
    (a--b) During RFT, similarity increases across all rollouts while decreasing among successful trajectories; SR denotes success rate.
    (c) Pass@\(K\) shows that the final RFT policy elicits success with fewer rollouts than SFT.
    (d) Coverage@\(K\), measured against SFT successful trajectories at \(K=256\), shows broader reference-repertoire coverage under limited sampling.
    }
    \label{fig:rft_exploration_dynamics}
\end{figure*}

\textbf{Global Contraction Masks Successful-Mode Diversification.}
Fig.~\ref{fig:rft_exploration_dynamics}(a--b) reveals a contrasting trend: normalized Global Alignment Kernel (GAK) similarity increases across all rollouts but decreases among successful ones as RFT progresses. Meanwhile, success rates rise from 41\% to 75\% for \(\pi_0\) and from 52\% to 77\% for \(\pi_{0.5}\). Thus, RFT concentrates the overall behavior distribution while improving task success and diversifying successful executions. Global contraction therefore masks successful-mode diversification rather than implying a collapse of task-valid exploration.

\textbf{RFT Makes Latent Successful Behaviors More Accessible.}
We measure the accessibility of successful behavior using Pass@\(K\), the probability that at least one of \(K\) rollouts succeeds for a task-initial-state configuration. Fig.~\ref{fig:rft_exploration_dynamics}(c) shows that RFT achieves substantially higher Pass@\(K\) at small \(K\), while the gap narrows as SFT saturates with larger budgets. This pattern suggests that successful behavior is often latent in SFT but difficult to elicit. RFT makes such behavior accessible with far fewer rollouts, converting sparse success into reliable execution under limited sampling.

\textbf{Elicitation Spans a Broad Task-Valid Solution Space.}
Pass@\(K\) does not distinguish between repeatedly sampling a few dominant modes and accessing a broader solution set. To measure breadth, we use SFT successes at \(K\!=\!256\) as a reference and compute Coverage@\(K\) through normalized GAK matching, with leave-one-out matching for SFT queries to avoid trivial self-matches and ensure a fair comparison. As shown in Fig.~\ref{fig:rft_exploration_dynamics}(d), RFT achieves substantially higher coverage at small \(K\) and matches SFT coverage with far fewer rollouts. Thus, RFT improves both the accessibility of success and the breadth of task-valid solutions available under limited sampling.

Taken together, RFT contracts behavior globally while expanding access to diverse task-valid solutions, revealing \textit{an underexplored opportunity}: broader successful-mode coverage may provide alternative strategies under distribution shifts. Since standard rewards do not distinguish among successful solutions, the central question behind DRIVE is \textit{whether explicitly incentivizing their diversity can broaden task-valid exploration and improve VLA generalization}.

\begin{figure*}[t]
    \centering
    \includegraphics[
    width=\textwidth,
    ]{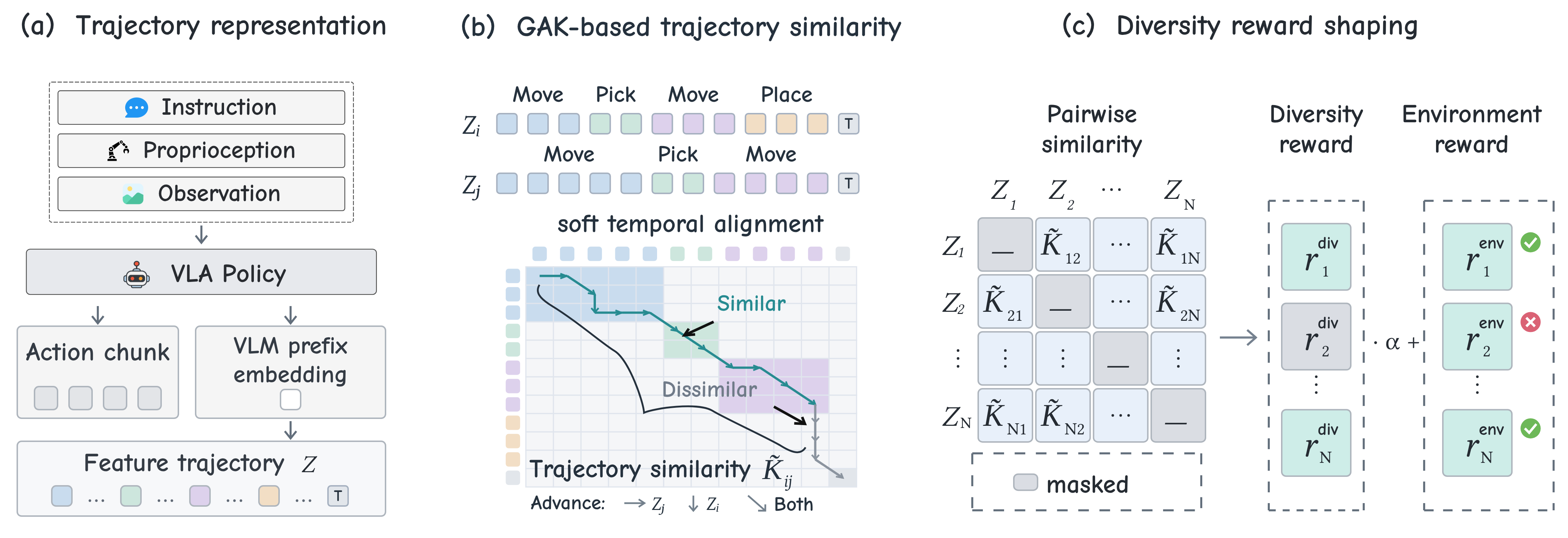}
        \caption{
        Overview of DRIVE.
        (a) Rollouts are encoded as per-decision feature trajectories.
        (b) GAK computes behavioral similarity by softly aligning feature sequences over monotonic temporal paths, handling execution-rate mismatch.
        (c) Pairwise similarities are converted into success-conditioned diversity rewards and added to environment rewards to promote exploration diversity.
        }
    \label{fig:main}
\end{figure*}

\section{More Ways to Succeed with DRIVE}
\label{sec:method}

In this section, we present DRIVE, a diversity-driven framework for VLA RFT. We first formulate the RFT objective and motivate behavior-aware rewards, then derive temporally aligned trajectory similarity from grouped rollouts, and finally convert relative successful-behavior diversity into a normalized intrinsic reward for success-conditioned shaping.

\subsection{Problem Formulation}
\label{sec:problem_formulation}
We consider a language-conditioned embodied control problem modeled as a Markov decision process (MDP) \(\mathcal{M}=(\mathcal{S},\mathcal{A},P_0,P,R,\gamma)\), where \(P_0\), \(P\), \(R\), and \(\gamma\) denote the initial-state distribution, environment dynamics, environment reward, and discount factor, respectively.
Given a language instruction \(l\), a VLA policy produces \(a_t \sim \pi_\theta(\cdot\mid s_t,l)\), where \(s_t\) includes visual and proprioceptive observations and \(a_t\) denotes either a single control action or an action chunk.
A rollout is \(\tau=(s_0,a_0,\ldots,s_{T-1},a_{T-1},s_T)\), and \(r_t^{\mathrm{env}}\) denotes the environment reward at step \(t\).

In RFT of VLA models, policy-gradient methods update the policy using rewards from environment rollouts to favor higher-return behaviors~\citep{li2026simplevla}. We adopt PPO~\citep{schulman2017ppo}, which constrains the update through a clipped importance ratio, with the surrogate objective
\begin{equation}
\mathcal{L}_{\mathrm{PPO}}(\theta)
=
\mathbb{E}_t
\left[
\min\left(
\rho_t(\theta)\hat{A}_t,\,
\operatorname{clip}\!\left(\rho_t(\theta),1-\epsilon,1+\epsilon\right)\hat{A}_t
\right)
\right],
\end{equation}
where \(\rho_t(\theta)=\pi_\theta(a_t\mid s_t,l)/\pi_{\theta_{\mathrm{old}}}(a_t\mid s_t,l)\) is the importance ratio and \(\hat{A}_t\) is the estimated advantage, computed from environment rewards and value estimates.
Because PPO advantages are derived from rewards, reward design determines which behaviors RFT reinforces. Standard task rewards capture accomplishment or progress but treat behaviorally distinct successes alike, motivating trajectory-level similarity as a basis for rewarding diverse task-valid solutions.

\subsection{Deriving Trajectory-Level Behavioral Similarity}
\label{sec:trajectory_similarity}

To distinguish different task-valid solution modes, we require a trajectory-level measure of behavioral similarity between rollouts.
We compare rollouts only within matched groups that share the same task and environment condition, so that the resulting similarity reflects differences in execution rather than differences in task specification or initial conditions.

We represent each trajectory as an ordered sequence of per-decision embeddings.
For rollout \(i\) with \(L_i\) executed action chunks, let \(z_{i,c}\) denote the feature representation of the observation at policy decision \(c\), for \(c=0,\ldots,L_i-1\).
In our implementation, we reuse the VLM-prefix representation already computed by the policy and mean-pool the valid prefix tokens into one embedding per policy decision. 
After the rollout terminates, we additionally extract the same representation from the terminal observation, yielding 
\begin{equation}
Z_i=(z_{i,0},\ldots,z_{i,L_i}),
\end{equation}
where \(z_{i,L_i}\) denotes the terminal-observation representation.

Behaviorally similar trajectories may progress at different rates, making step-wise comparison sensitive to \textit{temporal misalignment}.
We therefore use the Global Alignment Kernel (GAK)~\citep{cuturi2011gak}, which compares feature sequences while preserving temporal order and accommodating local timing differences.
Let \(G_{i,j}(t,u)\) denote the accumulated alignment similarity between the prefixes \(z_{i,0:t}\) and \(z_{j,0:u}\).
GAK computes this quantity through the dynamic-programming recurrence
\begin{equation}
G_{i,j}(t,u)
=
\kappa(z_{i,t},z_{j,u})
\left[
G_{i,j}(t-1,u)
+
G_{i,j}(t,u-1)
+
G_{i,j}(t-1,u-1)
\right],
\end{equation}
where \(\kappa\) denotes the local similarity between feature embeddings; in our implementation, we use a shifted cosine similarity mapped to \([0,1]\).
The three predecessor transitions allow either sequence, or both, to advance, accommodating local timing differences while preserving temporal order.
By accumulating over these alternatives, GAK accounts for all admissible monotonic alignments path rather than selecting a single best alignment.
After both trajectories have been fully traversed, the final dynamic-programming value gives their trajectory-level GAK kernel:
\begin{equation}
K_{\mathrm{GAK}}(Z_i,Z_j)
=
G_{i,j}(|Z_i|,|Z_j|).
\end{equation}
We normalize this kernel using the self-similarities of the two trajectories:
\begin{equation}
\widetilde K_{ij}
=
\frac{K_{\mathrm{GAK}}(Z_i,Z_j)}
{\sqrt{
K_{\mathrm{GAK}}(Z_i,Z_i)
K_{\mathrm{GAK}}(Z_j,Z_j)
}}.
\end{equation}
We use \(\widetilde K_{ij}\) as the normalized trajectory-level behavioral similarity between trajectories \(i\) and \(j\), where larger values indicate greater behavioral similarity.

\subsection{Success-Conditioned Diversity Shaping}
\label{sec:success_conditioned_novelty}

The trajectory similarities derived in Sec.~\ref{sec:trajectory_similarity} quantify behavioral differences but do not directly provide an optimization signal. We first convert them into a relative diversity score, then use this score to construct a success-conditioned potential and derive the DRIVE shaping reward.

\paragraph{Intra-Group Diversity.}
For a matched rollout group \(g\), let
\(\mathcal{T}_g=\{\tau_i\}_{i=1}^{N}\) contain \(N\) trajectories collected under the same task and environment condition. We define the average similarity $\text{sim}_i$ and relative diversity $d_i$ of trajectory \(\tau_i\) as
\begin{equation}
\text{sim}_i=
\frac{1}{N-1}
\sum\nolimits_{j\neq i}
\widetilde K_{ij}, \quad\quad d_i=1-\text{sim}_i.
\end{equation}
A larger \(d_i\) indicates that \(\tau_i\) is more behaviorally distinct from the other rollouts. We compare each trajectory against all group members because both successful and failed rollouts represent behaviors already explored by the policy.

Let \(y_i=\mathbf{1}[\tau_i\text{ succeeds}]\) and
\(\mathcal{S}_g=\{i\mid y_i=1\}\) denote the success indicator and the successful set within the group $g$. To stabilize the reward scale, we normalize diversity within $\mathcal{S}_g$:
\begin{equation}
\hat d_i=\operatorname{clip}\left(\frac{d_i-d_{\min}}{\max(d_{\max}-d_{\min},\delta)},0,1\right), \quad\text{where}~
d_{\min}=\min_{j\in\mathcal{S}_g}d_j,~~ d_{\max}=\max_{j\in\mathcal{S}_g}d_j,
\end{equation}
and \(\delta>0\) prevents negligible differences from being amplified. For \(i\notin\mathcal{S}_g\), we set \(\hat d_i=0\).

\textbf{Success-Conditioned Diversity Potential.}
Let \(s_{i,c}\) denote the state containing the prefix up to action chunk $c$ in trajectory $\tau_i\in g$. 
We define the diversity potential conditioned on task success as
\begin{equation}
    \Phi_g(s_{i,c})=y_i\cdot\hat{d}(s_{i,c}).
\end{equation}
Thus, behavioral diversity contributes potential only when the underlying trajectory successfully completes the task, preventing diverse failures from receiving an exploration bonus.

\textbf{Potential-Based Reward Shaping.}
We define the intrinsic reward as the potential difference between adjacent states as
\begin{equation}
    r^{\mathrm{div}}_{i,c}=\gamma\Phi_g(s_{i,c+1})-\Phi_g(s_{i,c}).
\end{equation}
Since the policy gradient operates at the trajectory level, the discounted return telescopes to
\begin{equation}
    r^{\mathrm{div}}_{\tau_i}=\sum\nolimits_{c=0}^{L_i-1}
\gamma^c r^{\mathrm{div}}_{i,c}=\gamma^{L_i}\Phi_g(s_{i,L_i})-\Phi_g(s_{i,0})=\gamma^{L_i}\cdot y_i\cdot\hat{d}_i,
\end{equation}
where the diversity of an initial state is zero, i.e., $\Phi_g(s_{i,0})=0$.
This formulation avoids the need to calculate the diversity of any intermediate prefixes within the trajectory.
Finally, DRIVE implements trajectory-level diversity bonus as a success-conditioned, potential-based shaping reward as
\begin{equation}
    r^{\mathrm{DRIVE}}_{\tau_i} = r^{\mathrm{env}}_{\tau_i}+\alpha\cdot r^{\mathrm{div}}_{\tau_i},
\end{equation}
where \(\alpha\) controls the strength of the diversity reward relative to the environment reward.

\begin{theorem}[Optimal Policy Invariance]\label{theo:shaping}
Let $M\!=\!(S,A,P,R,\gamma)$ denote the MDP for the VLA RFT task.
$\hat{d}(\cdot)\!:S\mapsto \mathbb{R}$ is a real-valued function that computes the trajectory-level diversity $\hat{d}(s)$ of the state $s$ within a group of rollouts.
We formulate $r^{\text{div}}(\cdot)\!:S\!\times\! A\!\times\! S\mapsto \mathbb{R}$ as an intrinsic reward function that is the difference between trajectory diversities of two adjacent states, such that for all $s\in S, a\in A, s'\in S$, $r^{\text{div}}(s,a,s')=\gamma y\hat{d}(s')-y\hat{d}(s)$, where \(y\in\{0,1\}\) is a binary success indicator. 
Then, with any constant balancing ratio $\alpha$, every optimal policy in the transformed MDP $M'\!=\!(S,A,P,R+\alpha R^{\text{div}},\gamma)$ will also be an optimal policy in $M$, and vice versa.
\end{theorem}

This shows that the success-conditioned diversity reward reshapes the learning signal without changing the set of optimal policies.
The proof is provided in Appendix~\ref{app:policy_invariance}.

\section{Experiments}
\label{sec:experiment}

We evaluate DRIVE from three complementary perspectives: generalization across distribution shifts, behavioral dynamics during RFT, and transfer to real-world deployment.


\textbf{Benchmarks.}
We first evaluate DRIVE on three manipulation benchmarks in simulation (see Sec.~\ref{sec:real_world} for real-robot deployment): LIBERO-Plus, ManiSkill3, and RoboTwin~2.0.
On LIBERO-Plus~\citep{fei2025liberoplus}, we construct IND (in-domain) and OOD (out-of-domain) conditions for the four LIBERO suites: Spatial, Object, Goal, and Long~\citep{liu2023libero}.
For each suite, the RFT/IND pool retains all ten base tasks and uses selected LIBERO-Plus perturbation values, while OOD evaluation uses held-out perturbation values over the same tasks.
On ManiSkill3~\citep{taomaniskill3}, we follow the generalization setting of RL4VLA~\citep{liu2025rlvla}, covering visual-language, semantic, and execution shifts.
On RoboTwin~2.0~\citep{chen2025robotwin}, we adapt the perturbation taxonomy of RoboTwin~2.0-Plus~\citep{zhang2026worldaction} to construct disjoint IND and OOD initialization sets.
Detailed task selections and split definitions are provided in Appendix~\ref{app:benchmark_details}.

\textbf{Baselines and Training.}
We compare DRIVE with Vanilla RFT and three representative exploration or regularization baselines.
\textbf{Vanilla RFT} follows the Flow-SDE-based RFT setup of $\pi_{\mathrm{RL}}$~\citep{chen2025pirl} and uses only the environment reward.
\textbf{Higher Noise} increases stochasticity during Flow-SDE sampling~\citep{chen2025pirl,zou2026precise}.
\textbf{KL Regularization} constrains the policy toward the SFT reference policy through a KL penalty~\citep{gxchen2026kl}.
\textbf{Clip-Higher} relaxes the upper PPO clipping bound following DAPO~\citep{yu2025dapo}.
These baselines intervene at the sampling, policy, and optimization levels, respectively, whereas DRIVE directly shapes successful-behavior diversity.
We evaluate both $\pi_0$~\citep{black2025pi0} and $\pi_{0.5}$~\citep{intelligence2025pi05}, with all methods initialized from the same SFT checkpoint and trained under a common PPO-based Flow-SDE RFT pipeline.
We report IND and OOD success rates, with detailed training hyperparameters and implementation settings provided in Appendix~\ref{app:training_details}.

\subsection{Main Results}
\label{sec:main_results}

\begin{table*}[t]
\centering
\caption{
Main results on simulation benchmarks. Values are task success rates (\%).
RoboTwin~2.0 averages over Click Bell and Press Stapler. Avg. is the macro-average over benchmark families.
Best and second-best results are bold and underlined.
}
\label{tab:main_results}

\small
\setlength{\tabcolsep}{2.2pt}
\renewcommand{\arraystretch}{1.08}

\begin{tabular}{l*{14}{c}}
\toprule
\multirow{2}{*}{Method}
& \multicolumn{2}{c}{\scriptsize\shortstack{LIBERO-Plus\\Spatial}}
& \multicolumn{2}{c}{\scriptsize\shortstack{LIBERO-Plus\\Object}}
& \multicolumn{2}{c}{\scriptsize\shortstack{LIBERO-Plus\\Goal}}
& \multicolumn{2}{c}{\scriptsize\shortstack{LIBERO-Plus\\Long}}
& \multicolumn{2}{c}{ManiSkill3}
& \multicolumn{2}{c}{RoboTwin~2.0}
& \multicolumn{2}{c}{Avg.} \\
\cmidrule(lr){2-3}
\cmidrule(lr){4-5}
\cmidrule(lr){6-7}
\cmidrule(lr){8-9}
\cmidrule(lr){10-11}
\cmidrule(lr){12-13}
\cmidrule(lr){14-15}
& IND & OOD
& IND & OOD
& IND & OOD
& IND & OOD
& IND & OOD
& IND & OOD
& IND & OOD \\
\midrule

\multicolumn{15}{c}{$\pi_0$} \\
\midrule

SFT
& 22.4 & 14.8
& 42.0 & 27.6
& 18.4 & 14.1
& 30.0 & 24.5
& 39.7 & 21.1
& 42.0 & 46.1
& 36.6 & 29.2 \\

Vanilla RFT
& \underline{78.8} & \underline{48.9}
& \underline{70.8} & \underline{53.4}
& \textbf{55.6} & \underline{37.9}
& \underline{66.0} & \underline{46.1}
& \underline{66.2} & \underline{35.5}
& \underline{65.9} & \underline{62.1}
& \underline{66.6} & \underline{48.1} \\

\textbf{DRIVE}
& \textbf{79.6} & \textbf{50.2}
& \textbf{74.4} & \textbf{53.5}
& \underline{54.8} & \textbf{39.8}
& \textbf{66.4} & \textbf{50.8}
& \textbf{68.4} & \textbf{40.3}
& \textbf{70.6} & \textbf{71.4}
& \textbf{69.3} & \textbf{53.4} \\

\midrule
\multicolumn{15}{c}{$\pi_{0.5}$} \\
\midrule

SFT
& 85.6 & 74.8
& 86.4 & 80.9
& 52.8 & 50.7
& 38.4 & 32.6
& 55.3 & 34.2
& 78.7 & 79.7
& 66.6 & 57.9 \\

Vanilla RFT
& 97.6 & 84.5
& 92.8 & 82.8
& 64.4 & \underline{57.9}
& 76.8 & 59.6
& 81.5 & 48.6
& 93.3 & 93.6
& 85.9 & 71.1 \\

Higher Noise
& 93.6 & 84.5
& \textbf{94.0} & \textbf{86.6}
& 60.4 & 56.7
& 74.0 & 58.2
& 81.8 & 45.9
& 93.3 & 93.0
& 85.2 & 70.1 \\

Clip-Higher
& 96.4 & 81.7
& 92.8 & 84.3
& 64.8 & 57.3
& \underline{77.6} & 59.9
& 80.0 & \textbf{52.9}
& 92.9 & \underline{94.0}
& 85.3 & \underline{72.6} \\

KL Reg.
& \underline{98.0} & \underline{86.3}
& 92.0 & 81.2
& \textbf{65.6} & \textbf{59.0}
& 74.4 & \underline{60.1}
& \underline{85.3} & 49.0
& \underline{93.5} & 93.4
& \underline{87.1} & 71.4 \\

\textbf{DRIVE}
& \textbf{98.8} & \textbf{88.1}
& \underline{93.2} & \underline{85.8}
& \underline{65.2} & 56.8
& \textbf{78.8} & \textbf{62.5}
& \textbf{86.8} & \underline{51.9}
& \textbf{93.8} & \textbf{94.2}
& \textbf{88.2} & \textbf{73.1} \\

\bottomrule
\end{tabular}
\end{table*}

Table~\ref{tab:main_results} summarizes the main simulation results across LIBERO-Plus, ManiSkill3, and RoboTwin~2.0.
We compare Vanilla RFT and DRIVE on both $\pi_0$ and $\pi_{0.5}$, with additional comparisons against Higher Noise, KL Regularization, and Clip-Higher on $\pi_{0.5}$.

DRIVE improves upon Vanilla RFT across both backbones and most benchmark splits.
On $\pi_0$, DRIVE improves 11 of the 12 benchmark-split results, including all six OOD results, increasing the benchmark-family macro-average from 66.6\% to 69.3\% under IND conditions and from 48.1\% to 53.4\% under OOD conditions.
The same pattern largely holds for $\pi_{0.5}$, where DRIVE again improves 11 of the 12 results and five of the six OOD results, raising the macro-average from 85.9\% to 88.2\% under IND conditions and from 71.1\% to 73.1\% under OOD conditions.
Although the OOD gain is larger on $\pi_0$ than on $\pi_{0.5}$, the improvements span the four LIBERO-Plus suites, ManiSkill3, and RoboTwin~2.0.
This consistency across two VLA backbones and heterogeneous manipulation settings indicates that the benefit of successful-behavior diversity shaping is not confined to a particular model or task distribution,  but instead reflects a broader effect that persists across different model architectures and evaluation settings.

The broader comparison on $\pi_{0.5}$ shows that alternative exploration and regularization strategies can improve individual benchmarks, but their gains vary across conditions.
Higher Noise, KL Regularization, and Clip-Higher each outperform Vanilla RFT in selected settings, yet none provides consistently stronger performance across the evaluated benchmarks, suggesting that these alternatives offer condition-specific gains rather than a uniform advantage across evaluation conditions.
DRIVE achieves the highest benchmark-family macro-average under both IND (88.2\%) and OOD (73.1\%) evaluation while remaining competitive across individual benchmarks.
Overall, these results show that explicitly shaping successful-behavior diversity provides broad performance gains across heterogeneous manipulation settings and improves generalization under distribution shift.

\subsection{Generalization and Behavioral Dynamics}
\label{sec:analysis}

We further examine how DRIVE's generalization evolves during RFT and whether its performance gains are accompanied by changes in behavioral diversity.
We first track IND and OOD performance across training checkpoints, then compare trajectory-level behavioral similarity over all rollouts and successful rollouts, and finally isolate the role of success-conditioned diversity shaping.

\begin{figure*}[t]
    \centering
    \includegraphics[width=\textwidth]{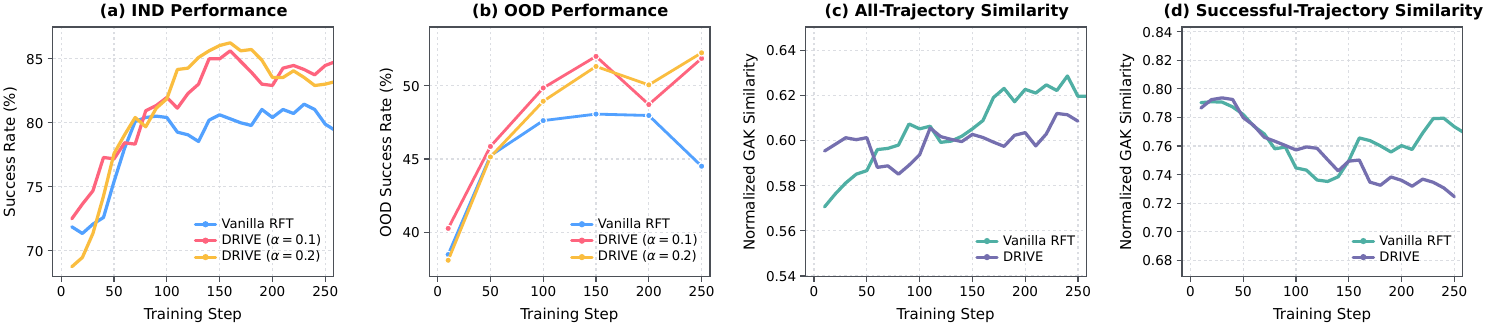}
    \caption{
    Generalization and behavioral dynamics during RFT for $\pi_{0.5}$ on ManiSkill3.
    (a) IND success rate across training steps.
    (b) Average OOD success rate across visual-language, semantic, and execution shifts at selected training steps for Vanilla RFT and DRIVE with two shaping strengths.
    (c) Normalized GAK similarity over all trajectories for Vanilla RFT and DRIVE.
    (d) Normalized GAK similarity over successful trajectories for Vanilla RFT and DRIVE.
    Behavioral analyses in (c--d) use DRIVE with $\alpha=0.1$.
    Lower similarity indicates greater behavioral diversity.
    }
    \label{fig:generalization_diversity}
\end{figure*}

\paragraph{Performance evolution during RFT.}
Figure~\ref{fig:generalization_diversity}(a--b) shows the evolution of IND performance throughout RFT and OOD performance at selected evaluation steps for $\pi_{0.5}$ on ManiSkill3, comparing Vanilla RFT with DRIVE under two shaping strengths, \(\alpha=0.1\) and \(\alpha=0.2\).
Under IND evaluation, both DRIVE variants improve rapidly and achieve stronger performance than Vanilla RFT over substantial portions of training.
Under OOD evaluation, performance is less monotonic and fluctuates more across checkpoints, yet both DRIVE variants achieve higher success rates than Vanilla RFT at multiple evaluated checkpoints.
The two shaping strengths exhibit broadly similar performance trends across both IND and OOD evaluation, showing that the observed improvement is consistent across the tested values of \(\alpha\).
The non-monotonic OOD trajectories further indicate that continued optimization on the RFT distribution does not necessarily translate into steadily improving generalization. This discrepancy also suggests that IND performance alone provides an incomplete characterization of how the policy generalizes during RFT.
Evaluating both distributions throughout training is therefore important for distinguishing improvements on the fine-tuning distribution from changes in robustness to distribution shift.

\begingroup
\hyphenpenalty=5000
\tolerance=1000
\emergencystretch=1em

\begin{wrapfigure}{r}{0.45\textwidth}
    \centering
    \vspace{-6pt}
    \includegraphics[width=\linewidth]{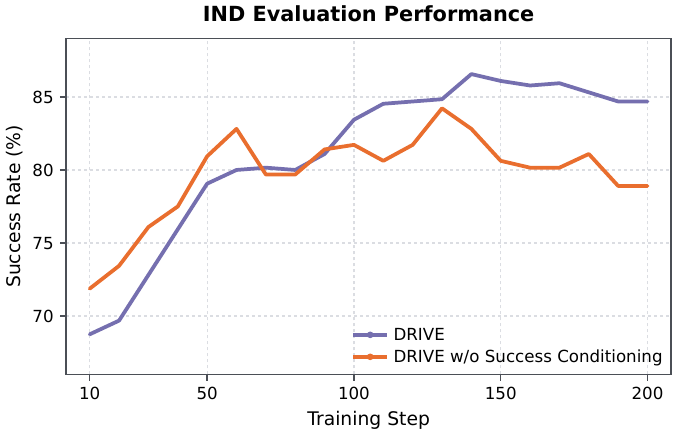}
    \caption{Effect of success conditioning on IND performance. IND evaluation success rates of DRIVE with and without success conditioning throughout RFT.}
    \label{fig:success_conditioning}
    \vspace{-8pt}
\end{wrapfigure}

\noindent\textbf{Successful-behavior diversity.}
We next examine whether the performance differences are accompanied by changes in behavioral diversity.
Using the normalized GAK similarity introduced in Sec.~\ref{sec:trajectory_similarity}, Fig.~\ref{fig:generalization_diversity}(c) compares all rollout trajectories, while Fig.~\ref{fig:generalization_diversity}(d) considers only successful trajectories.
Across all rollouts, Vanilla RFT and DRIVE exhibit broadly similar similarity trends, with only a moderate separation emerging later in training.
The difference becomes substantially clearer among successful trajectories, where DRIVE maintains lower similarity during the later stages of RFT, indicating greater diversity among task-successful behaviors.
This contrast shows that the behavioral effect of DRIVE is concentrated within successful trajectories rather than expressed as uniformly broader variation across the entire rollout distribution.
Consequently, diversity measured over all rollouts can obscure changes that occur specifically within the task-valid solution space.
Together with the stronger OOD performance in Fig.~\ref{fig:generalization_diversity}(b), these results suggest that DRIVE's generalization gains are associated with greater diversity among successful behaviors, consistent with the motivation behind success-conditioned diversity shaping.

\endgroup

\noindent\textbf{Isolating the role of success conditioning.}
We finally isolate the role of success conditioning by comparing full DRIVE with a variant that applies the same diversity shaping without restricting the reward to successful trajectories.
As shown in Fig.~\ref{fig:success_conditioning}, the two variants exhibit comparable IND performance during the early stages of RFT, with no consistent advantage for either method.
As training progresses, however, full DRIVE achieves higher success rates and the separation becomes clearer in the later stages.
The delayed emergence of this gap suggests that success conditioning does not simply strengthen the optimization signal uniformly throughout training.
Rather, the later separation suggests that restricting the diversity bonus to task-valid trajectories becomes more important as RFT progresses, avoiding explicit rewards for behavioral variation that does not lead to task success.
These results indicate that diversity shaping is more effective when conditioned on task success than when applied indiscriminately across all trajectories.

\subsection{Real-World Experiments}
\label{sec:real_world}

Finally, we evaluate whether the generalization gains of DRIVE persist under physical deployment.
We conduct experiments on a dual-arm platform equipped with two AgileX PiPER-X manipulators.
To reduce the simulation-to-real discrepancy, we perform hand--eye and robot-coordinate calibration, reproduce the relative arm placement and manipulation workspace in simulation, align the camera viewpoints, and adjust the scene geometry and illumination to better match the physical environment.
Additional details of the physical setup, simulation alignment, evaluation protocol, and representative rollout sequences are provided in Appendix~\ref{app:real_world}.

We evaluate two manipulation tasks, \emph{Click Bell} and \emph{Press Stapler}.
Since the physical robot configuration differs from the original benchmark setup, we first collect expert trajectories in the aligned simulation environment to obtain a common SFT initialization, from which Vanilla RFT and DRIVE are trained separately.
For real-world evaluation, each policy is tested under a clean condition and three distribution shifts involving distractor objects, tabletop-background changes, and lighting variations.
Each task--condition pair is evaluated over 20 trials.
Figure~\ref{fig:real_world_scenes} shows representative evaluation conditions, while Table~\ref{tab:real_world_results} reports the corresponding success rates.

\begin{figure}[!t]
    \centering
    \begin{minipage}[t]{0.38\textwidth}
        \vspace{0pt}
        \centering
        \begin{minipage}[c][3.6cm][c]{\linewidth}
            \centering
            \includegraphics[width=\linewidth]{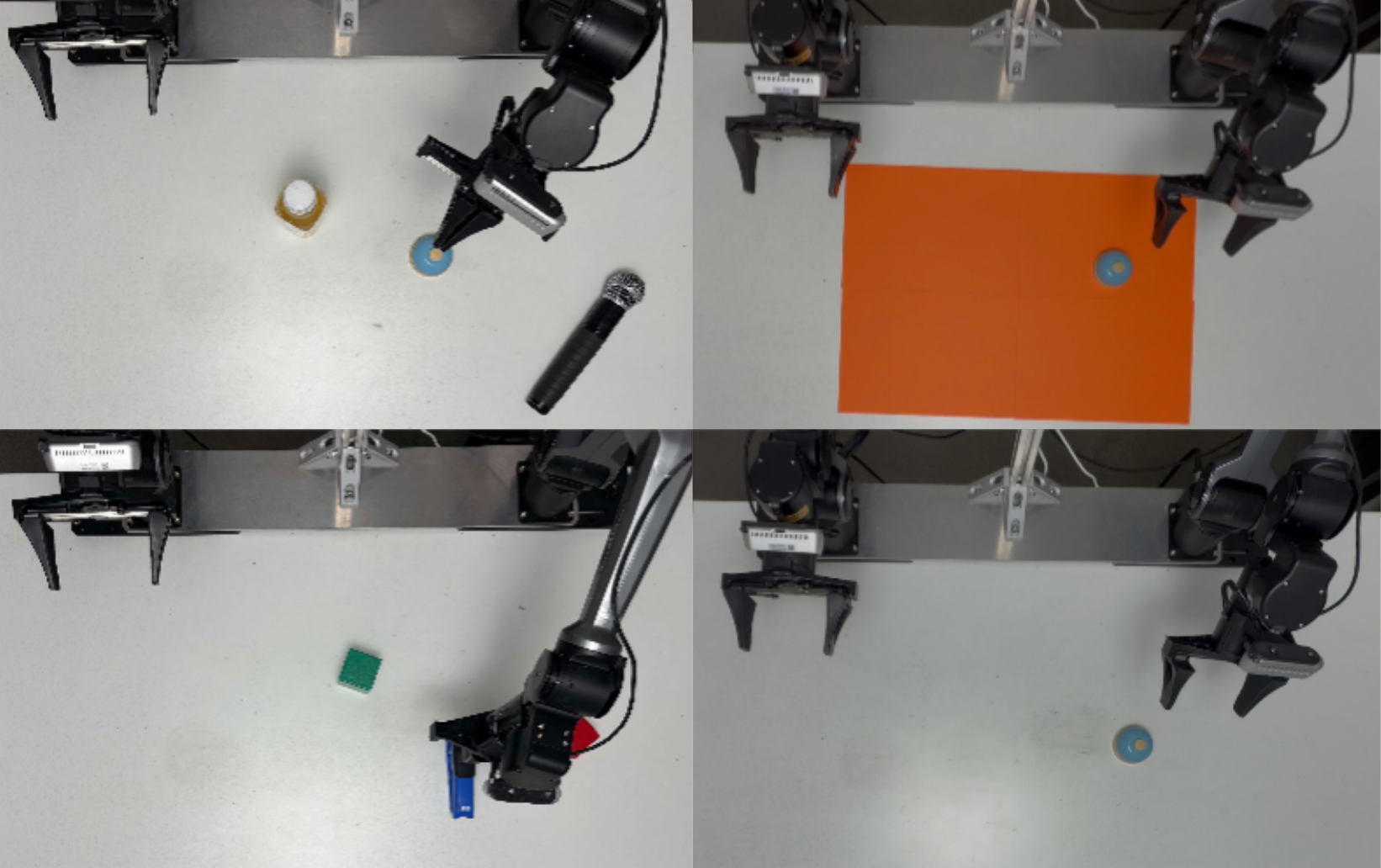}
        \end{minipage}
        \captionof{figure}{
        Representative real-world evaluation conditions: Clean, Distractor, Background, and Lighting.
        }
        \label{fig:real_world_scenes}
    \end{minipage}
    \hfill
    \begin{minipage}[t]{0.58\textwidth}
        \vspace{0pt}
        \centering
        \begin{minipage}[c][3.6cm][c]{\linewidth}
            \centering
            \scriptsize
            \setlength{\tabcolsep}{1.8pt}
            \renewcommand{\arraystretch}{1.15}
            \begin{tabular}{llccccc}
                \toprule
                Task & Method & Clean & Distractor & Background & Lighting & OOD Avg. \\
                \midrule
                \multirow{2}{*}{Click Bell}
                & Vanilla RFT & 80 & 60 & 70 & \textbf{65} & 65.0 \\
                & \textbf{DRIVE} & \textbf{85} & \textbf{75} & \textbf{75} & \textbf{65} & \textbf{71.7} \\
                \midrule
                \multirow{2}{*}{Press Stapler}
                & Vanilla RFT & 75 & 70 & 65 & 55 & 63.3 \\
                & \textbf{DRIVE} & \textbf{80} & \textbf{75} & \textbf{80} & \textbf{70} & \textbf{75.0} \\
                \bottomrule
            \end{tabular}
        \end{minipage}
        \captionof{table}{
        Real-world success rates (\%) over 20 trials per task-condition pair.
        OOD Avg. denotes the mean over Distractor, Background, and Lighting.
        }
        \label{tab:real_world_results}
    \end{minipage}
\end{figure}

DRIVE consistently improves performance under physical deployment.
Across the eight task--condition combinations, DRIVE achieves higher success rates in seven settings and matches Vanilla RFT in the remaining one.
Averaged across the two tasks, the clean-condition success rate increases from 77.5\% to 82.5\%.
The improvement is larger under distribution shift, where the average OOD success rate increases from 64.1\% to 73.3\%, corresponding to a 9.2 percentage-point gain over Vanilla RFT.
For Click Bell, the largest improvement occurs under Distractor ($+15$ points), while for Press Stapler the largest gains occur under Background and Lighting (both $+15$ points).
In comparison, DRIVE improves the clean-condition success rate by $5$ points on both tasks, suggesting that its advantage is more pronounced under distribution shift.
These results are consistent with the simulation experiments and further show that the generalization gains of DRIVE persist when transferred from simulation to the physical robotic system.

\section{Conclusions, Limitations, and Future Work}

We show that VLA RFT contracts behavior globally while expanding the accessibility and diversity of task-valid successful modes. Motivated by this finding, DRIVE uses potential-based, success-conditioned diversity shaping to encourage broader solution coverage without rewarding diverse failures. Across three simulation benchmarks, DRIVE improves the OOD macro-average over Vanilla RFT by 5.3 points on \(\pi_0\) and 2.0 points on \(\pi_{0.5}\), with a further 9.2-point gain on a dual-arm AgileX PiPER-X platform. These results establish successful-mode coverage as both a useful lens for understanding RFT and a practical objective for improving VLA generalization.

In DRIVE, pairwise GAK comparisons within matched rollout groups introduce additional computation and rely on VLM representations to capture meaningful behavioral differences. Future work could develop more efficient, learned diversity metrics that better distinguish task-solving strategies from incidental motion. Moreover, extending DRIVE to longer-horizon tasks, broader embodiments, and online real-world fine-tuning would further assess its scalability and generality.

\bibliography{iclr2027_conference}
\bibliographystyle{iclr2027_conference}

\clearpage

\appendix

\clearpage
\addtocontents{toc}{\protect\setcounter{tocdepth}{2}}

\renewcommand{\contentsname}{\LARGE Appendix}
\tableofcontents

\clearpage

\section{Optimal Policy Invariance of DRIVE}
\label{app:policy_invariance}

Following~\citet{ng1999policy}, we provide the proof of Theorem~\ref{theo:shaping}, showing that the success-conditioned diversity signal preserves optimal policies when introduced through potential-based reward shaping.

\begin{proof}[Proof of Theorem~\ref{theo:shaping}]
For a fixed rollout group, define the success-conditioned diversity potential
\[
\Phi(s)
\triangleq
y\hat d(s),
\]
where \(y\in\{0,1\}\) is the binary success indicator and
\(\hat d(s)\) is the normalized trajectory-level diversity. Since
\(\hat d(s)\in[0,1]\), the potential \(\Phi(s)\) is bounded. The intrinsic reward in Theorem~\ref{theo:shaping} can then be written as
\[
r^{\mathrm{div}}(s,a,s')
=
\gamma\Phi(s')-\Phi(s)
=
\gamma y\hat d(s')-y\hat d(s).
\]

Let \(Q_M^*(s,a)\) denote the optimal action-value function of the original MDP
\(M=(S,A,P,R,\gamma)\). It satisfies the Bellman optimality equation
\[
Q_M^*(s,a)
=
\mathbb{E}_{s'\sim P(\cdot\mid s,a)}
\left[
r(s,a,s')
+
\gamma\max_{a'\in A}Q_M^*(s',a')
\right].
\]
Subtracting \(\alpha\Phi(s)\) from both sides and adding and subtracting
\(\alpha\gamma\Phi(s')\) inside the expectation gives
\[
Q_M^*(s,a)-\alpha\Phi(s)=
\mathbb{E}_{s'}
\Big[r(s,a,s')+\alpha\bigl(\gamma\Phi(s')-\Phi(s)\bigr)
+\gamma\max_{a'\in A}
\bigl(Q_M^*(s',a')-\alpha\Phi(s')\bigr)\Big].
\]
The equality holds because \(\Phi(s')\) depends only on the successor state and is therefore independent of the maximization over \(a'\).

Define
\(
\widetilde Q(s,a)
\triangleq
Q_M^*(s,a)-\alpha\Phi(s).
\)
Using
\(
r^{\mathrm{div}}(s,a,s')
=
\gamma\Phi(s')-\Phi(s),
\)
we obtain
\[
\widetilde Q(s,a)
=
\mathbb{E}_{s'\sim P(\cdot\mid s,a)}
\left[
r(s,a,s')
+\alpha r^{\mathrm{div}}(s,a,s')
+\gamma\max_{a'\in A}\widetilde Q(s',a')
\right].
\]
This is precisely the Bellman optimality equation for the transformed MDP
\[
M'
=
(S,A,P,R+\alpha R^{\mathrm{div}},\gamma).
\]
By the uniqueness of the Bellman optimality fixed point,
\[
Q_{M'}^*(s,a)
=
\widetilde Q(s,a)
=
Q_M^*(s,a)-\alpha\Phi(s).
\]

Because the subtracted term \(\alpha\Phi(s)\) depends only on the state and not on the action,
\[
\begin{aligned}
\arg\max_{a\in A}Q_{M'}^*(s,a)
&=
\arg\max_{a\in A}
\left[
Q_M^*(s,a)-\alpha\Phi(s)
\right]\\
&=
\arg\max_{a\in A}Q_M^*(s,a).
\end{aligned}
\]
Therefore, the sets of optimal actions are identical at every state. Every optimal policy in \(M'\) is consequently optimal in \(M\), and every optimal policy in \(M\) is also optimal in \(M'\). Hence, DRIVE's success-conditioned potential-based shaping preserves the set of optimal policies.
\end{proof}

\section{Trajectory Representation and Global Alignment Similarity}
\label{app:gak_details}

This section provides additional details on the trajectory representation and temporally aligned similarity used in Sec.~\ref{sec:trajectory_similarity}.

\paragraph{Trajectory representation.}
For each policy decision, we extract a feature \(z_{i,t}\) from the VLM prefix representation already computed by the policy, using mean pooling over valid prefix tokens. If trajectory \(\tau_i\) contains \(L_i\) action chunks, we additionally extract a feature from the terminal observation and represent the trajectory as
\[
Z_i=
\left(
z_{i,0},\ldots,z_{i,L_i-1},z_{i,\mathrm{term}}
\right),
\qquad
M_i\coloneqq |Z_i|=L_i+1.
\]
Thus, each trajectory is retained as an ordered variable-length feature sequence rather than collapsed into a single embedding.

\paragraph{Local similarity and temporal alignment.}
For two features, we use a shifted cosine similarity
\[
k_{ij}(t,u)
=
\operatorname{clip}
\left(
\frac{
\widehat z_{i,t}^{\top}\widehat z_{j,u}+1
}{2},
\epsilon,
1
\right),
\]
where \(\widehat z\) denotes the \(\ell_2\)-normalized feature and \(\epsilon=10^{-12}\).

To accommodate different execution speeds and trajectory lengths, we compute the Global Alignment Kernel (GAK)~\cite{cuturi2011gak}. Let \(G_{ij}(a,b)\) denote the alignment kernel between the first \(a\) features of \(Z_i\) and the first \(b\) features of \(Z_j\). With \(G_{ij}(0,0)=1\) and the remaining boundary entries set to zero,
\[
G_{ij}(a,b)
=
k_{ij}(a-1,b-1)
\left[
G_{ij}(a-1,b)
+
G_{ij}(a,b-1)
+
G_{ij}(a-1,b-1)
\right].
\]
The three transitions allow either trajectory, or both, to advance while preserving temporal order. Unlike a single best-path alignment, this recurrence aggregates contributions from all admissible monotonic alignment paths.

\paragraph{Normalized similarity.}
The unnormalized trajectory kernel is
\[
K_{\mathrm{GAK}}(Z_i,Z_j)
=
G_{ij}(|Z_i|,|Z_j|).
\]
and we use the normalized similarity
\[
\widetilde K_{ij}
=
\frac{
K_{\mathrm{GAK}}(Z_i,Z_j)
}{
\sqrt{
K_{\mathrm{GAK}}(Z_i,Z_i)
K_{\mathrm{GAK}}(Z_j,Z_j)
}
}.
\]
The implementation evaluates the dynamic program in log space for numerical stability. The resulting \(\widetilde K_{ij}\in[0,1]\) is used as the trajectory-level behavioral similarity throughout the paper.

\section{Benchmark and Evaluation Details}
\label{app:benchmark_details}

\subsection{LIBERO-Plus}
\label{app:libero_plus}

We evaluate all four standard LIBERO suites---Spatial, Object, Goal, and
Long---using the task variations provided by LIBERO-Plus
\citep{fei2025liberoplus,liu2023libero}.
Our split preserves all ten original base tasks in each suite and separates
IND and OOD conditions through the variation values applied to these tasks,
rather than by holding out base-task identities.
This keeps the underlying manipulation semantics fixed while testing
generalization to changes in scene configuration, visual appearance, language,
and robot/environment state.

\paragraph{IND construction.}
For each suite, the RFT/IND pool contains all ten clean LIBERO tasks together
with 240 selected LIBERO-Plus variants, yielding 250 task definitions per
suite.
We construct the IND variation pool from four perturbation families that can
be partitioned by discrete configurations or identities:
\emph{Confounding Objects}, \emph{Target Object Pose},
\emph{Scene Theme}, and \emph{Surface Appearance}.
The selected variants cover all ten base tasks and remain fixed throughout
training.
Table~\ref{tab:libero_ind_composition} summarizes the resulting composition.

For Spatial, Object, and Goal, the selected variants are distributed across
the ten base tasks to maintain broad task coverage while exposing the policy
to multiple values from each IND perturbation family.
For Long, which contains longer-horizon and more fragile manipulation
sequences, we retain the same overall pool size but use a more conservative
allocation of layout perturbations and a larger proportion of background
variations.
The Spatial, Object, and Goal splits sample uniformly from their complete
sets of selected task definitions.
For LIBERO Long, however, the number of available definitions differs across
base tasks, so we use base-task-uniform hierarchical sampling.
Specifically, we first sample one of the ten base tasks uniformly, then sample
uniformly from the selected clean and Plus task definitions associated with
that base task, and finally sample a valid initial-state index for the selected
definition.
This keeps the marginal sampling frequency of the Long base tasks balanced
despite differences in the number of available Plus variants.

\paragraph{OOD construction.}
OOD evaluation uses perturbation values that are excluded from the IND/RFT
pool.
For object-layout perturbations, we separate distractor identities and
target-object configurations between IND and OOD.
Similarly, scene themes and surface appearances are partitioned by their
underlying identities or configurations to avoid overlap between the two
splits.
The remaining LIBERO-Plus dimensions---camera viewpoints, robot initial
states, language instructions, lighting conditions, and sensor
noise---provide additional held-out distribution shifts for OOD evaluation.
Thus, IND and OOD retain the same underlying manipulation tasks while differing
in the perturbation values that define the corresponding observation and
execution conditions.

LIBERO-Plus organizes its variations into seven major perturbation families:
\emph{Objects Layout}, \emph{Camera Viewpoints},
\emph{Robot Initial States}, \emph{Language Instructions},
\emph{Light Conditions}, \emph{Background Textures}, and
\emph{Sensor Noise}.
Across the four suites, the complete LIBERO-Plus manifests contain
2,402 Spatial, 2,518 Object, 2,591 Goal, and 2,519 Long variants.
Table~\ref{tab:libero_plus_statistics} summarizes their distribution across
perturbation families.

\begin{table*}[h]
    \centering
    \small
    \setlength{\tabcolsep}{7pt}
    \renewcommand{\arraystretch}{1.12}
    \caption{
    Composition of the LIBERO-Plus IND task pools.
    Each suite contains the ten original clean LIBERO tasks and 240 selected
    Plus variants.
    }
    \label{tab:libero_ind_composition}
    \begin{tabular}{lrrrrrr}
        \toprule
        Suite
        & Clean
        & Confounding
        & Target Pose
        & Scene Theme
        & Surface
        & Total \\
        \midrule
        Spatial
        & 10 & 78 & 66 & 44 & 52 & 250 \\
        Object
        & 10 & 77 & 72 & 37 & 54 & 250 \\
        Goal
        & 10 & 81 & 63 & 42 & 54 & 250 \\
        Long
        & 10 & 65 & 50 & 52 & 73 & 250 \\
        \bottomrule
    \end{tabular}
\end{table*}

\begin{table*}[h]
    \centering
    \small
    \setlength{\tabcolsep}{8pt}
    \renewcommand{\arraystretch}{1.12}
    \caption{
    LIBERO-Plus perturbation statistics for the four evaluated suites.
    Counts denote variants in the complete Plus manifests and are distinct
    from the selected IND subsets in
    Table~\ref{tab:libero_ind_composition}.
    }
    \label{tab:libero_plus_statistics}
    \begin{tabular}{lrrrr}
        \toprule
        Perturbation family
        & Spatial
        & Object
        & Goal
        & Long \\
        \midrule
        Objects Layout
        & 385 & 403 & 425 & 312 \\
        Camera Viewpoints
        & 376 & 396 & 408 & 419 \\
        Robot Initial States
        & 350 & 398 & 409 & 393 \\
        Language Instructions
        & 390 & 354 & 410 & 383 \\
        Light Conditions
        & 292 & 297 & 279 & 274 \\
        Background Textures
        & 258 & 248 & 281 & 289 \\
        Sensor Noise
        & 351 & 422 & 379 & 449 \\
        \midrule
        \textbf{Total}
        & \textbf{2,402}
        & \textbf{2,518}
        & \textbf{2,591}
        & \textbf{2,519} \\
        \bottomrule
    \end{tabular}
\end{table*}

\paragraph{Evaluation protocol.}
We evaluate each task definition using exactly one fixed initial state.
Specifically, we always select the first available initialization associated
with that task definition.
Original LIBERO tasks may provide multiple initialization states
(e.g., 50 initializations), whereas many LIBERO-Plus variants provide only a
single state; in both cases, only the first initialization is used for
evaluation.
Consequently, every evaluated task definition contributes exactly one episode
to the reported metric, independent of the number of initialization states
provided by the benchmark.
The same fixed task--initialization pairs are used for all compared methods
and VLA backbones.
Performance is reported as the mean binary task success rate over the
corresponding IND or OOD evaluation set.

\FloatBarrier
\subsection{ManiSkill3}
\label{app:maniskill}

We evaluate ManiSkill3~\citep{taomaniskill3} using the
\texttt{PutOnPlateInScene25Main-v3} pick-and-place environment, following the
generalization setting of RL4VLA~\citep{liu2025rlvla}.
Unlike LIBERO-Plus, which contains multiple semantic base tasks, the ManiSkill3
evaluation uses a common manipulation objective and constructs distribution
shifts by modifying the visual observation, semantic grounding requirements,
or execution conditions.
The IND environment contains 16 training objects, 16 training scenes, and
randomized object and plate initial states.

\paragraph{OOD construction.}
We organize the held-out evaluation conditions into three groups:
\emph{Visual-Language OOD}, \emph{Semantic OOD}, and
\emph{Execution OOD}.
Visual-Language OOD changes the appearance of the visual input or the language
expression while preserving the underlying manipulation objective.
The six settings consist of an instruction variation,
\texttt{Instruct-v1}; an unseen background appearance,
\texttt{VisionImage-v1}; and four texture perturbations,
\texttt{VisionTexture03-v1}, \texttt{VisionTexture05-v1},
\texttt{VisionWhole03-v1}, and \texttt{VisionWhole05-v1}.
The texture variants apply either local foreground or global image-level
perturbations at two strengths.

Semantic OOD modifies the grounding problem while retaining the same basic
pick-and-place skill.
We consider four settings:
\texttt{MultiCarrot-v1 (train)} and
\texttt{MultiCarrot-v1 (test)}, which introduce multiple candidate objects,
and \texttt{MultiPlate-v1 (train)} and
\texttt{MultiPlate-v1 (test)}, which introduce multiple candidate
receptacles.
The train variants use identities from the training object/receptacle sets,
whereas the test variants introduce held-out identities.
These settings therefore test whether the policy can resolve the target object
or receptacle from language in scenes containing semantically competing
alternatives.

Execution OOD changes the state distribution or environment during task
execution rather than primarily modifying appearance or semantic content.
The three settings cover shifts in initial object/plate configurations and
dynamic changes during an episode.
For example, \texttt{Position-v1} changes the initial spatial distribution,
while \texttt{PositionChangeTo-v1} changes the target-object position during
execution, requiring the policy to react to a modified environment.
Together, these cases evaluate spatial-state generalization and closed-loop
recovery under execution-time distribution shifts.

\FloatBarrier
\begin{table*}[h]
    \centering
    \small
    \setlength{\tabcolsep}{8pt}
    \renewcommand{\arraystretch}{1.12}
    \caption{
    ManiSkill3 evaluation protocol.
    Each setting is evaluated over 320 episodes.
    The OOD benchmark contains 13 settings grouped into visual-language,
    semantic, and execution shifts.
    }
    \label{tab:maniskill_eval_protocol}
    \begin{tabular}{llrr}
        \toprule
        Split / Category
        & Evaluation focus
        & Settings
        & Episodes \\
        \midrule
        IND
        & Training objects, scenes, and state distribution
        & 1 & 320 \\
        \midrule
        Visual-Language OOD
        & Visual appearance and language variation
        & 6 & 1,920 \\
        Semantic OOD
        & Object/receptacle grounding under semantic ambiguity
        & 4 & 1,280 \\
        Execution OOD
        & Initial-state shifts and execution-time changes
        & 3 & 960 \\
        \midrule
        \textbf{Total}
        & \textbf{IND + three OOD categories}
        & \textbf{14}
        & \textbf{4,480} \\
        \bottomrule
    \end{tabular}
\end{table*}

\paragraph{Evaluation protocol.}
Each evaluation setting is run for 320 episodes using a fixed evaluation
configuration shared by all compared methods and backbones.
This yields 1,920 episodes for Visual-Language OOD, 1,280 for Semantic OOD,
and 960 for Execution OOD, for a total of 4,160 OOD episodes in addition to
the 320 IND episodes.
Success is determined by the environment's binary task-completion criterion.
We first average success over the episodes belonging to each evaluation
setting and then aggregate the settings within their corresponding OOD
category.
The resulting IND and OOD success rates are used for the comparisons reported
in the main paper.

\FloatBarrier
\subsection{RoboTwin 2.0}
\label{app:robotwin}

We evaluate RoboTwin~2.0~\citep{chen2025robotwin} on two dual-arm
manipulation tasks, \emph{Click Bell} and \emph{Press Stapler}.
To evaluate generalization beyond the nominal task distribution, we construct
task-specific IND and OOD initialization sets following the structured
perturbation design of RoboTwin~2.0-Plus~\citep{zhang2026worldaction}.
The underlying task objective remains unchanged across the two splits; the
distribution shift is induced through controlled changes to objects, scene
appearance, camera and lighting conditions, robot state, language, and visual
observations.

\paragraph{IND/OOD construction.}
For each manipulation task, we construct a set of 1,500 evaluation
configurations.
The first 1,000 configurations, indexed by $[0,1000)$, define the IND set,
while the remaining 500, indexed by $[1000,1500)$, define the OOD set.
The IND set contains clean configurations together with variations in object
configuration, background appearance, and robot initial state.
The OOD set introduces held-out values from these dimensions and additionally
includes shifts in lighting, camera viewpoint, language instruction, and
sensor observations.
Table~\ref{tab:robotwin_split} summarizes the composition used for each task.

\FloatBarrier
\begin{table*}[!t]
    \centering
    \small
    \setlength{\tabcolsep}{9pt}
    \renewcommand{\arraystretch}{1.12}
    \caption{
    RoboTwin~2.0 IND/OOD construction for each evaluated task.
    The same split structure is independently instantiated for
    \emph{Click Bell} and \emph{Press Stapler}.
    }
    \label{tab:robotwin_split}
    \begin{tabular}{lrrl}
        \toprule
        Perturbation family
        & IND
        & OOD
        & Description \\
        \midrule
        Clean
        & 50 & 0
        & Nominal task configurations \\
        Objects
        & 134 & 34
        & Object additions and pose variations \\
        Background
        & 750 & 150
        & Scene/background appearance variations \\
        Lighting
        & 0 & 50
        & Held-out illumination conditions \\
        Camera
        & 0 & 50
        & Held-out camera configurations \\
        Robot State
        & 66 & 16
        & Robot initial-state variations \\
        Language
        & 0 & 50
        & Held-out instruction variations \\
        Sensor
        & 0 & 150
        & Visual/sensor perturbations \\
        \midrule
        \textbf{Total}
        & \textbf{1,000}
        & \textbf{500}
        & \textbf{1,500 initializations per task} \\
        \bottomrule
    \end{tabular}
\end{table*}

\paragraph{Perturbation generation.}
The manifest combines explicitly validated physical-state configurations with
deterministically derived observation-level perturbations.
For each task, 300 configurations are grounded in expert-validated scene
states: 250 belong to IND, consisting of 50 clean configurations,
67 object-addition variants, 67 object-pose variants, and 66 robot-state
variants; the remaining 50 are OOD anchors, consisting of 17 object-addition,
17 object-pose, and 16 robot-state variants.
The other 1,200 configurations are generated deterministically from validated
scene states by modifying background, lighting, camera, language, or sensor
conditions.
This construction allows appearance-level shifts to reuse a physically valid
underlying manipulation scene rather than requiring a separate expert
demonstration for every rendered variation.

The object perturbations contain two primary subtypes:
object additions (\texttt{O1}) and target/object pose changes
(\texttt{O2}).
The remaining perturbations operate primarily on scene appearance,
observation, instruction, or robot initialization.
All perturbations are instantiated as single-factor variations according to
the manifest definition, allowing the source of each distribution shift to be
controlled during evaluation.

\paragraph{Evaluation protocol.}
For each task, the full IND pool contains 1,000 evaluation configurations.
We uniformly sample a fixed subset of 350 configurations from this pool for
evaluation and reuse exactly the same subset across all compared methods and
VLA backbones.
For OOD evaluation, we use the full set of 500 predefined OOD configurations.
The same perturbation assignments are used for all compared methods and
backbones.

Each selected configuration produces one rollout and is scored using the
binary task-success criterion defined by the RoboTwin environment.
We compute success rates separately for \emph{Click Bell} and
\emph{Press Stapler}, and report their macro-average as the overall
RoboTwin~2.0 result in the main table.

\section{Training and Implementation Details}
\label{app:training_details}

This section provides additional details of the RFT training protocol,
rollout collection, DRIVE implementation, and hyperparameters used in our
experiments.
Unless otherwise specified, all methods within the same benchmark--backbone
setting are initialized from the same SFT checkpoint and use the same
training horizon, rollout distribution, and PPO optimization configuration.
Training horizons are selected per benchmark--backbone setting according to
task difficulty and convergence behavior and are held fixed across the
compared methods.

\paragraph{Common RFT setup.}
All experiments use PPO with Flow-SDE sampling and are initialized from the
same SFT checkpoint within each benchmark--backbone setting.
We use $\gamma=0.99$, GAE $\lambda=0.95$, and PPO clip ratio $0.2$ throughout.
Benchmark-specific optimization, rollout, and action-chunk parameters are
reported in Tables~\ref{tab:libero_training_hparams} and
\ref{tab:other_training_hparams}.
Here, $H$ denotes the action prediction horizon and $H'$ the number of actions
executed before replanning.

\paragraph{Rollout collection.}
We sample 64 groups with 8 stochastic rollouts per group, yielding 512
trajectories per collection.
Rollouts within a group share the same task and environment condition, and
DRIVE computes trajectory similarity only within each group.

\paragraph{DRIVE implementation.}
For each rollout, we construct a temporal feature sequence from the VLM-prefix
representations associated with the executed action chunks.
Valid token features are aggregated using mean pooling, and the terminal
observation is included when constructing the trajectory representation.
Within each rollout group, we compute pairwise normalized Global Alignment
Kernel (GAK) similarity between trajectories as described in
Sec.~\ref{sec:trajectory_similarity}.

The resulting relative diversity signal is used as an intrinsic shaping term.
Importantly, the shaping reward is success-conditioned: only successful
trajectories receive the diversity bonus, preventing behaviorally unusual
failures from being explicitly rewarded.
The shaping term is applied at the terminal action chunk.
The coefficient $\alpha$ controls the strength of this term and is held
constant throughout training; we do not use an $\alpha$ decay schedule in
the reported experiments.
Benchmark-specific values are listed in
Tables~\ref{tab:libero_training_hparams} and
\ref{tab:other_training_hparams}.

\begin{table*}[t]
    \centering
    \scriptsize
    \setlength{\tabcolsep}{3.8pt}
    \renewcommand{\arraystretch}{1.08}
    \caption{
    Training hyperparameters for LIBERO-Plus.
    Runner epochs and DRIVE shaping coefficients correspond to the final
    experimental configurations.
    All compared methods within each column use the same runner horizon.
    }
    \label{tab:libero_training_hparams}
    \begin{tabular}{lcccccccc}
        \toprule
        & \multicolumn{4}{c}{$\pi_0$}
        & \multicolumn{4}{c}{$\pi_{0.5}$} \\
        \cmidrule(lr){2-5}
        \cmidrule(lr){6-9}
        Parameter
        & Spatial & Object & Goal & Long
        & Spatial & Object & Goal & Long \\
        \midrule

        Runner epochs
        & 200 & 300 & 400 & 200
        & 300 & 250 & 100 & 300 \\

        Global batch size
        & 2048 & 2048 & 2048 & 2048
        & 2048 & 2048 & 2048 & 2048 \\

        PPO update epochs
        & 4 & 4 & 4 & 4
        & 3 & 3 & 3 & 4 \\

        Actor learning rate
        & $5{\times}10^{-6}$ & $5{\times}10^{-6}$
        & $5{\times}10^{-6}$ & $5{\times}10^{-6}$
        & $5{\times}10^{-6}$ & $5{\times}10^{-6}$
        & $5{\times}10^{-6}$ & $5{\times}10^{-6}$ \\

        Critic learning rate
        & $1{\times}10^{-4}$ & $1{\times}10^{-4}$
        & $1{\times}10^{-4}$ & $1{\times}10^{-4}$
        & $1{\times}10^{-4}$ & $1{\times}10^{-4}$
        & $1{\times}10^{-4}$ & $1{\times}10^{-4}$ \\

        LR schedule
        & Const. & Const. & Const. & Const.
        & Const. & Const. & Const. & Cosine \\

        \midrule

        Reward discount $\gamma$
        & 0.99 & 0.99 & 0.99 & 0.99
        & 0.99 & 0.99 & 0.99 & 0.99 \\

        GAE $\lambda$
        & 0.95 & 0.95 & 0.95 & 0.95
        & 0.95 & 0.95 & 0.95 & 0.95 \\

        PPO clip $\epsilon$
        & 0.20 & 0.20 & 0.20 & 0.20
        & 0.20 & 0.20 & 0.20 & 0.20 \\

        \midrule

        Episode horizon
        & 240 & 240 & 320 & 480
        & 240 & 240 & 320 & 480 \\

        Sampling groups
        & 64 & 64 & 64 & 64
        & 64 & 64 & 64 & 64 \\

        Rollouts / group
        & 8 & 8 & 8 & 8
        & 8 & 8 & 8 & 8 \\

        Trajectories / collection
        & 512 & 512 & 512 & 512
        & 512 & 512 & 512 & 512 \\

        \midrule

        Prediction horizon $H$
        & 50 & 50 & 50 & 50
        & 10 & 10 & 10 & 10 \\

        Execution horizon $H'$
        & 10 & 10 & 10 & 10
        & 10 & 10 & 10 & 10 \\

        Flow-SDE denoise steps
        & 4 & 4 & 4 & 4
        & 3 & 5 & 5 & 5 \\

        Flow-SDE noise $\sigma$
        & 0.5 & 0.5 & 0.5 & 0.5
        & 0.5 & 0.3 & 0.3 & 0.5 \\

        \midrule

        DRIVE $\alpha$
        & 0.10 & 0.10 & 0.05 & 0.05
        & 0.10 & 0.10 & 0.10 & 0.10 \\

        \bottomrule
    \end{tabular}
\end{table*}

\begin{table*}[t]
    \centering
    \scriptsize
    \setlength{\tabcolsep}{4.0pt}
    \renewcommand{\arraystretch}{1.08}
    \caption{
    Training hyperparameters for ManiSkill3 and RoboTwin~2.0.
    Training horizons are shared by all compared methods within each
    benchmark--backbone setting.
    }
    \label{tab:other_training_hparams}
    \begin{tabular}{lcccccc}
        \toprule
        &
        \multicolumn{2}{c}{ManiSkill3}
        &
        \multicolumn{2}{c}{Click Bell}
        &
        \multicolumn{2}{c}{Press Stapler} \\
        \cmidrule(lr){2-3}
        \cmidrule(lr){4-5}
        \cmidrule(lr){6-7}
        Parameter
        & $\pi_0$ & $\pi_{0.5}$
        & $\pi_0$ & $\pi_{0.5}$
        & $\pi_0$ & $\pi_{0.5}$ \\
        \midrule

        Runner epochs
        & 300 & 200
        & 100 & 100
        & 100 & 50 \\

        Global batch size
        & 5120 & 5120
        & 2048 & 2048
        & 2048 & 2048 \\

        PPO update epochs
        & 4 & 5
        & 5 & 5
        & 5 & 5 \\

        Actor learning rate
        & $7.91{\times}10^{-6}$ & $7.91{\times}10^{-6}$
        & $5.6{\times}10^{-6}$ & $5{\times}10^{-6}$
        & $5.6{\times}10^{-6}$ & $5{\times}10^{-6}$ \\

        Critic learning rate
        & $1.55{\times}10^{-4}$ & $1.55{\times}10^{-4}$
        & $1.1{\times}10^{-4}$ & $1{\times}10^{-4}$
        & $1.1{\times}10^{-4}$ & $1{\times}10^{-4}$ \\

        LR schedule
        & Const. & Const.
        & Const. & Const.
        & Const. & Const. \\

        \midrule

        Reward discount $\gamma$
        & 0.99 & 0.99
        & 0.99 & 0.99
        & 0.99 & 0.99 \\

        GAE $\lambda$
        & 0.95 & 0.95
        & 0.95 & 0.95
        & 0.95 & 0.95 \\

        PPO clip $\epsilon$
        & 0.20 & 0.20
        & 0.20 & 0.20
        & 0.20 & 0.20 \\

        \midrule

        Episode horizon
        & 80 & 80
        & 200 & 200
        & 200 & 200 \\

        Sampling groups
        & 64 & 64
        & 64 & 64
        & 64 & 64 \\

        Rollouts / group
        & 8 & 8
        & 8 & 8
        & 8 & 8 \\

        Trajectories / collection
        & 512 & 512
        & 512 & 512
        & 512 & 512 \\

        \midrule

        Prediction horizon $H$
        & 8 & 8
        & 50 & 50
        & 50 & 50 \\

        Execution horizon $H'$
        & 5 & 5
        & 20 & 20
        & 20 & 20 \\

        Flow-SDE denoise steps
        & 4 & 4
        & 4 & 5
        & 4 & 5 \\

        Flow-SDE noise $\sigma$
        & 0.5 & 0.5
        & 0.5 & 0.5
        & 0.5 & 0.5 \\

        \midrule

        DRIVE $\alpha$
        & 0.10 & 0.10
        & 0.02 & 0.05
        & 0.05 & 0.05 \\

        \bottomrule
    \end{tabular}
\end{table*}

\paragraph{Baseline implementations.}
The auxiliary baselines use the same underlying RFT pipeline and modify only
the component associated with the corresponding exploration or regularization
strategy.
Vanilla RFT uses the standard Flow-SDE sampling configuration, PPO clipping,
and environment reward without additional diversity shaping.
Higher Noise increases the Flow-SDE sampling noise to $\sigma=0.7$.
KL Regularization adds a KL penalty toward the SFT reference policy with
coefficient $\beta=0.5$.
Clip-Higher keeps the lower PPO clipping bound at $0.20$ while increasing the
upper bound from $0.20$ to $0.28$.
All other benchmark-specific settings follow the corresponding Vanilla RFT
configuration.

\begin{table}[t]
    \centering
    \small
    \setlength{\tabcolsep}{5pt}
    \renewcommand{\arraystretch}{1.08}
    \caption{
    Method-specific modifications relative to Vanilla RFT.
    }
    \label{tab:baseline_training_details}
    \begin{tabular}{ll}
        \toprule
        Method & Modification \\
        \midrule
        Vanilla RFT
        & Standard PPO + Flow-SDE \\

        Higher Noise
        & Flow-SDE noise $\sigma=0.7$ \\

        KL Regularization
        & KL penalty $\beta=0.5$ \\

        Clip-Higher
        & PPO upper clip $0.20 \rightarrow 0.28$ \\

        DRIVE
        & Success-conditioned GAK shaping \\
        \bottomrule
    \end{tabular}
\end{table}

\section{Real-World Setup and Evaluation Details}
\label{app:real_world}

This section provides additional details of the physical platform,
simulation-to-real alignment, evaluation protocol, and representative
real-world rollouts used in Sec.~\ref{sec:real_world}.

\paragraph{Physical platform.}
Our real-world experiments are conducted on a dual-arm manipulation platform
equipped with two AgileX PiPER-X manipulators.
The workspace contains the two robot arms, wrist-mounted observation cameras,
a planar manipulation surface, and the task objects used for
\emph{Click Bell} and \emph{Press Stapler}.
Because this hardware configuration differs from the original RoboTwin setup,
we construct a corresponding simulation environment whose robot placement,
workspace geometry, and observation viewpoints are aligned with the physical
system before collecting demonstrations and performing RFT.

\paragraph{Geometric and camera calibration.}
We use a metric planar ChArUco board as the common calibration frame for the
dual-arm workspace.
The left and right manipulators are calibrated independently using the same
procedure and their respective wrist-mounted cameras.
For each arm, the end effector is moved through 48 predefined poses spanning
different positions and orientations in the reachable workspace.
At each pose, the robot state provides the base-to-end-effector transform,
while ChArUco detection followed by \texttt{solvePnP} estimates the
board-to-camera transform from the RGB observations.
Multiple valid estimates collected at the same pose are fused before
calibration.

The resulting multi-pose observations are passed to OpenCV's Shah
robot-world/hand-eye solver, which jointly recovers the pose of each robot
base in the common workspace frame and the rigid transform between the
end effector and its wrist-mounted camera.
For calibration pose $i$, the transformations satisfy

\begin{equation}
    {}^{W}\mathbf{T}_{B}\,
    {}^{B}\mathbf{T}_{E}^{(i)}\,
    {}^{E}\mathbf{T}_{C}
    =
    {}^{W}\mathbf{T}_{C}^{(i)},
    \label{eq:real_world_calibration}
\end{equation}

where $W$, $B$, $E$, and $C$ denote the ChArUco/workspace, robot-base,
end-effector, and wrist-camera coordinate frames, respectively.
Calibrating both arms with respect to the same workspace frame also provides
their relative placement.

\paragraph{Simulation alignment.}
We construct the corresponding RoboTwin setup to match the geometry and
observation configuration of the physical dual-arm system.
The calibrated physical workspace provides a common reference for the
relative robot arrangement, while the RoboTwin embodiment uses the
corresponding PiPER-X asset configuration and inter-arm spacing.
We additionally reproduce the manipulation-plane geometry and relative task
layout to match the real setup.

For camera alignment, the real-camera intrinsics are transferred to the
simulated observation cameras and adjusted to the rendered image resolution.
The wrist-camera extrinsics are not applied a second time from the hand-eye
calibration; instead, they remain defined by the fixed PiPER-X URDF asset used
by the RoboTwin embodiment.
This avoids double-applying the camera-to-robot transform while preserving
consistent wrist-camera viewpoints.
We further match the static observation viewpoint and adjust scene
illumination and visual appearance to reduce the remaining sim-to-real
discrepancy.

\paragraph{Task initialization and training.}
We evaluate two manipulation tasks, \emph{Click Bell} and
\emph{Press Stapler}.
Since the physical robot configuration differs from the original benchmark
configuration, we first collect expert trajectories in the aligned simulation
environment and use them to obtain a common SFT initialization.
Vanilla RFT and DRIVE are then trained independently from the same SFT
checkpoint in the aligned simulation environment.
The resulting policies are evaluated on the physical platform using the same
observation and action interface.

\paragraph{Real-world evaluation protocol.}
For each task, we evaluate one clean condition together with three distribution
shifts: \emph{Distractor}, \emph{Background}, and \emph{Lighting}.
The Distractor condition introduces additional objects into the workspace;
the Background condition changes the tabletop appearance; and the Lighting
condition changes scene illumination while preserving the underlying task.
These conditions probe whether the policy remains robust to visual and
contextual changes beyond the nominal physical setup.

Each task--condition pair is evaluated over 20 trials, and the same evaluation
conditions are used for Vanilla RFT and DRIVE.
A rollout is considered successful when the corresponding physical
manipulation objective is completed.
The quantitative results are reported in
Table~\ref{tab:real_world_results}.

\paragraph{Representative physical rollouts.}
Figure~\ref{fig:real_world_rollouts} provides qualitative examples of
successful physical executions under the four evaluation conditions.
For clarity, we visualize \emph{Click Bell}, for which the task objective
remains identical while the surrounding visual context changes.
Each row contains five temporally ordered frames sampled from a single rollout,
showing the progression from the initial scene to successful interaction with
the target.

\begin{figure*}[!t]
    \centering

    \includegraphics[width=\linewidth,keepaspectratio]
    {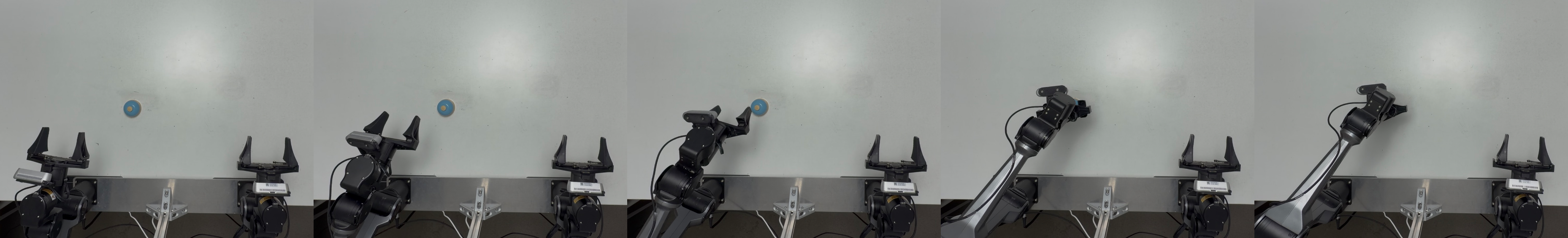}

    \vspace{2pt}

    \includegraphics[width=\linewidth,keepaspectratio]
    {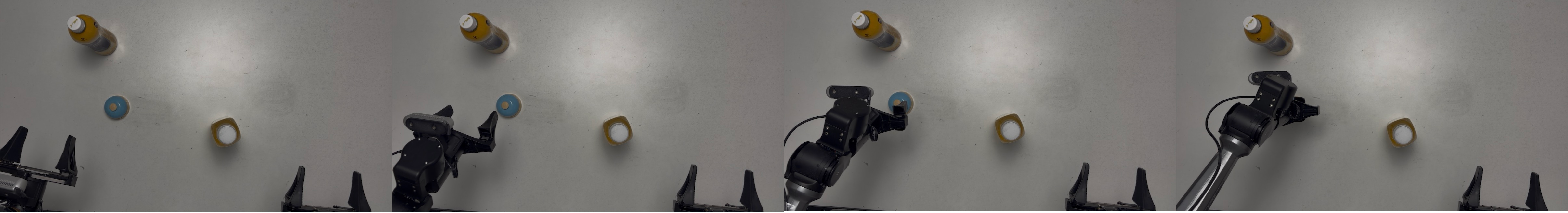}

    \vspace{2pt}

    \includegraphics[width=\linewidth,keepaspectratio]
    {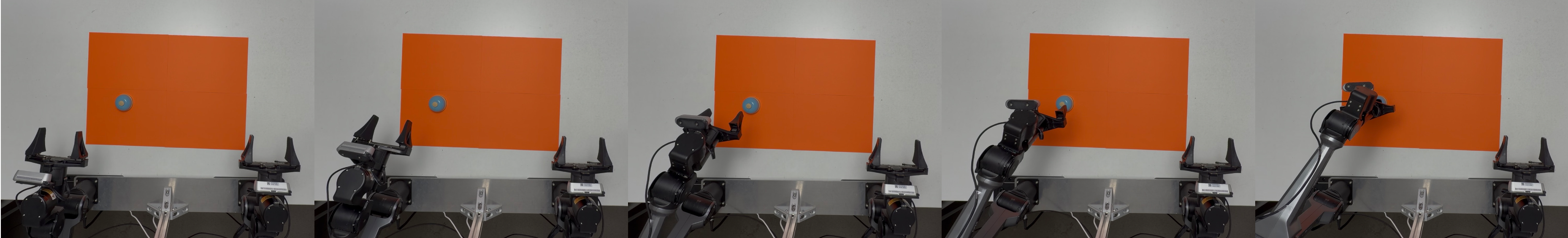}

    \vspace{2pt}

    \includegraphics[width=\linewidth,keepaspectratio]
    {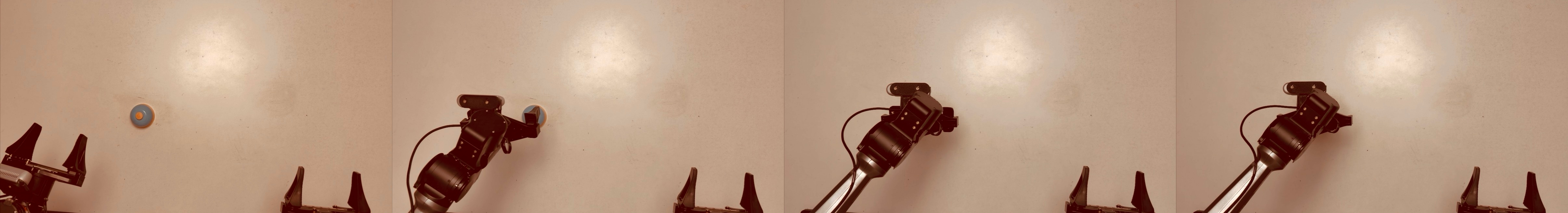}

    \caption{
    Representative successful real-world \emph{Click Bell} rollouts under
    Clean, Distractor, Background, and Lighting conditions (top to bottom).
    Each row shows a temporally ordered sequence of selected frames from one
    physical execution.
    }
    \label{fig:real_world_rollouts}
\end{figure*}

\end{document}